\documentclass{article}
\usepackage{ijcai26}
\usepackage{times}
\usepackage{soul}
\usepackage{url}
\usepackage[utf8]{inputenc}
\usepackage[small]{caption}
\usepackage{graphicx}
\usepackage{amsmath}
\usepackage{amsthm}
\usepackage{booktabs}
\usepackage{algorithm}
\usepackage{algorithmic}
\usepackage[switch]{lineno}

\usepackage{amsxtra, amsfonts, amssymb, amstext, amsmath, mathtools}
\usepackage{amsthm}
\usepackage{booktabs}
\usepackage{nicefrac}
\usepackage{xspace}
\usepackage{url}
\usepackage{graphics,color}
\usepackage[algo2e,ruled,vlined,linesnumbered]{algorithm2e}\SetArgSty{upshape}
\usepackage{wrapfig}
\usepackage{lmodern}
\usepackage[hidelinks]{hyperref}

\newtheorem{theorem}{Theorem}
\newtheorem{lemma}[theorem]{Lemma}

\newtheorem{definition}[theorem]{Definition}

\newcommand{\oea}{\mbox{${(1 + 1)}$~EA}\xspace}

\newcommand{\NSGA}{\mbox{NSGA}\nobreakdash-II\xspace}
\newcommand{\NSGAT}{\mbox{NSGA}\nobreakdash-III\xspace}
\newcommand{\SMS}{\mbox{SMS-EMOA}\xspace}
\newcommand{\SPEA}{\mbox{SPEA2}\xspace}

\newcommand{\onemax}{\textsc{OneMax}\xspace}
\newcommand{\LO}{\textsc{Leading\-Ones}\xspace}
\newcommand{\leadingones}{\LO}

\newcommand{\oneminmax}{\textsc{OneMinMax}\xspace}
\newcommand{\goneminmax}{\textsc{G-OneMinMax}\xspace}

\newcommand{\lotz}{\textsc{LOTZ}\xspace}

\newcommand{\glotz}{\textsc{G-LOTZ}\xspace}
\newcommand{\glo}{\textsc{G-LO}\xspace}
\newcommand{\gtz}{\textsc{G-TZ}\xspace}

\newcommand{\pmin}{p_{\min}}

\newcommand{\R}{\ensuremath{\mathbb{R}}}

\newcommand{\N}{\ensuremath{\mathbb{N}}} 
\newcommand{\Z}{\ensuremath{\mathbb{Z}}}

\let\originalleft\left
\let\originalright\right
\renewcommand{\left}{\mathopen{}\mathclose\bgroup\originalleft}
\renewcommand{\right}{\aftergroup\egroup\originalright}

\title{First Mathematical Runtime Analyses of Multi-Objective Evolutionary Algorithms for Multi-Valued Decision Variables}

\author{
Mingfeng Li$^1$
\and
Zheng Cheng$^1$ \and
Weijie Zheng$^{1,2}$\footnote{Corresponding author.}\and
Benjamin Doerr$^{3}$\\
\affiliations
$^1$School of Computer Science and Technology, National Key Laboratory of Smart Farm Technologies and Systems, International Research Institute for Artificial Intelligence, 
\\
Harbin Institute of Technology, Shenzhen\\
$^{2}$Pengcheng Laboratory, Shenzhen, China \\
$^3$Laboratoire d'Informatique (LIX), CNRS, \'Ecole Polytechnique, \\
Institut Polytechnique de Paris, Palaiseau, France\\
\emails
limingfeng@stu.hit.edu.cn,
zhengweijie@hit.edu.cn,
lastname@lix.polytechnique.fr
}

\begin{document}
\sloppy{

\maketitle

\begin{abstract}
Problems defined on binary decision spaces have been intensively studied in the theory of multi-objective evolutionary algorithms (MOEAs). In contrast, no mathematical runtime analyses exist so far for MOEAs dealing with decision variables that take a finite number \(r > 2\) of values, despite the prevalence of such problems in practice. In this work, we begin to fill this research gap. We analyze how the classic SEMO algorithm with unit-strength local mutation computes the Pareto front of an \(r\)-valued counterpart of the classic \oneminmax benchmark. For the expected number of function evaluations until the Pareto front is covered by the population of this MOEA, we prove an upper bound of \(O(n^2 r^2 \log n)\) and a near-tight lower bound of \(\Omega(n^2 r (r + \log n))\). We can close the small remaining gap between these two bounds by considering a variant of the algorithm that accepts only strictly better solutions; for this variant, we show an upper bound of \(O(n^2 r (r + \log n))\), matching our lower bound (which also holds for this variant).
Our results suggest that classic MOEAs encounter no significant additional difficulties when dealing with multi-valued decision variables. However, significantly more advanced tools may be required to obtain tight bounds for algorithms with more complex population dynamics.
\end{abstract}

\section{Introduction}

Evolutionary algorithms are widely used for solving complex optimization problems, in particular, for those with multiple conflicting objectives.
Multi-objective evolutionary algorithms (MOEAs) address such problems by exploring the different trade-offs between the objectives in a single run. Well-known MOEAs such as \NSGA, \NSGAT, \SPEA, and \SMS are widely used in both research and applications. Their empirical success has motivated a considerable body of theoretical works aimed at understanding how MOEAs work and why they perform well in practice. These include, among others, guarantees on the time required to cover the entire Pareto front \cite{LaumannsTZWD02,Giel03,ZhengD23aij,WiethegerD23,ZhengD24,RenBLQ24}, analyses on the benefits of crossover operators \cite{DoerrQ23crossover,DangOS24,Opris26}, the role of diversity preservation and archiving mechanisms \cite{
OsunaGNS20,BianRLQ24,RenLLQ25
}, the impact of non-elitist strategies \cite{ZhengLDD24,LiZD25,BianZLQ25,BendahiDFL25}, 
and approximation guarantees for situations in which the full Pareto front is too costly to be computed \cite{HorobaN08,ZhengD25approx,DengZD25,AlghouassDKL25,LiZZD25}. 

The vast majority of the theoretical works (in discrete optimization) have focused on problems defined on binary decision spaces. 
To the best of our knowledge, only two runtime analyses of MOEAs move beyond binary decision spaces, and they do so by considering the other extreme of unbounded integer-valued search domains~\cite{Rudolph23,DoerrKR25}. Among them, Rudolph~\shortcite{Rudolph23} analyzed subproblems of a multi-objective problem in integer search spaces, but the analysis remains a single-objective analysis. Following this work, Doerr et al. \shortcite{DoerrKR25} studied evolutionary algorithms on unbounded integer domains, focusing on how different mutation strength distributions affect the runtime. Till now, no mathematical runtime analyses exist for MOEAs dealing with decision variables that take a finite number $r$ of values that is larger than two. 

\paragraph{Our Contributions:}
In this work, we undertake a first attempt to analyze MOEAs for problems with multi-valued decision variables with mathematical means. We first extend the classical bi-objective \oneminmax benchmark to an $r$-valued domain, $r > 2$. This benchmark preserves the simplicity and structural clarity of its binary counterpart, but (naturally) has a larger Pareto front. 
We then consider the classic SEMO algorithm with unit-strength local mutation, which can be regarded as the natural multi-valued analogue of one-bit mutation. For this algorithm, we prove an upper bound of $O(n^2 r^2 \log n)$ and a lower bound of $\Omega(n^2 r (r+\log n))$ on the expected number of iterations until the population covers the entire Pareto front. The (small) gap between these bounds stems from the population dynamics, which (as this work shows) are much more complex in the multi-valued setting.

We further consider a very mild modification of the tie-breaking rule in the selection of the next population, namely that we accept only strictly better solutions (that is, in case of equal function values, we keep the old solution and discard the new one). This modification allows for a more tractable analysis of the population dynamics, and  allows us to prove an improved upper bound of $O(n^2 r (r+\log n))$, which matches our lower bound (that is valid also for this modification of the algorithm). 

We do not see any reason why the modified algorithm should be faster than the original one. Our (simple) experimental results also could not detect a significant difference (rather the modified algorithm appeared slightly slower). 
{To check that this observation is not specific to $r$-valued \oneminmax, where all solutions are Pareto optimal, we additionally introduce and experimentally study an $r$-valued analogue of \lotz. This benchmark has the same Pareto-front size, but, among many other differences, most solutions are not Pareto-optimal.
Again, we observe no noticeable runtime difference between the original and the modified algorithm.}
For these reasons,  we conjecture that the $O(n^2 r (r+\log n))$ bound {, proven for the modified SEMO on $r$-valued \oneminmax and matching our lower bound,} is also the true runtime for the original SEMO. Proving this conjecture, unfortunately, seems out of reach with the current methods of this field, and might be a worthwhile challenge for the future, that could give deep insights into the more complex population dynamics in multi-valued search spaces.


\section{Related Work}

The theory of randomized search heuristics is a relatively young subarea of artificial intelligence. While first mathematical runtime analyses for single-objective algorithms appeared already in the 1990s, say \cite{Muhlenbein92,Back93}, the first such works for multi-objective algorithms took around another ten years \cite{LaumannsTZWD02
} and the truly practically relevant algorithms could only be analyzed very recently \cite{LiZZZ16,ZhengD23aij,WiethegerD23,BianZLQ25,RenBLQ24}.

As stated already in the introduction, there are no theoretical works analyzing how an MOEA solves a problem defined over decision variables taking a finite number $r > 2$ of values. Closest to our work are the two recent papers \cite{Rudolph23,DoerrKR25} that analyze discrete multi-objective problems defined on $\Z^n$. Only the second work considers a true MOEA, namely the SEMO regarded also in this work and the related GSEMO. The main difficulty in this work, however, is dealing with the boundedness of the set of values taken by the decision variables. Consequently, the methods employed there are of little help for our problem setting. 

In contrast, in the theory of single-objective evolutionary algorithms a few sporadic works on multi-valued optimization exist. Doerr et al. \shortcite{DoerrJS11} analyzed the \oea with uniform mutation strength, that is, resetting a component to a random other value, on linear functions over $\{0,1,2\}^n$. They proved an expected runtime of $O(n\log n)$, matching the well-known bound for the binary case. The result was then extended by Doerr and Pohl \shortcite{DoerrP12}, where they studied $r$-valued linear functions and derived an upper bound of $O(rn \log n)+O(r^3 n\log \log n)$ and a lower bound of $\Omega(rn \log n)$. 

Subsequent work considered different and more complex mutation operators and in return only regarded the simpler \onemax problem. K\"otzing et al. \shortcite{KotzingLW15} analyzed the \oea with unit-strength mutation on dynamically changing \onemax where each dimension has $r$ different possible values. Doerr et al. \shortcite{DoerrDK18} considered the \oea with three different mutation strengths and proved asymptotically tight runtimes for a variety of \onemax-type test functions over an alphabet of size~$r$. They further analyzed the performance of randomized local search, a variant of the \oea using local mutation, with a self-adjusting mutation strength and proved an expected optimization time of $\Theta(n(\log n+ \log r))$, which is optimal among all comparison-based search heuristics. Harder et al.\shortcite{HarderKLRR24IntegerValued} studied the \oea with different mutation operators on the unbounded integer search space $\mathbb{Z}^n$. All these works are hardly comparable to ours since they regard single-objective optimization via an algorithm with parent and offspring population equal to one. Hence the population dynamics are less interesting here. 

In \cite{QianSTZ18submodular,QianZTY18}, submodular functions with multi-valued discrete domains are studied. Both works formulate a constrained single-objective problem as a bi-objective problem. They then use problem-specific algorithms and algorithms similar to classic MOEAs to solve these. However, since their objective is to find a single solution with a particular constraint satisfaction (equivalent to a particular value of the second objective), they do not analyze how the full Pareto front is computed or approximated, but only how long it takes until such a particular solution is contained in the population. For that reason, we could not derive deeper insights on multi-objective optimization from these works, and feel that they should be viewed as single-objective optimization works.

More recently, estimation-of-distribution algorithms (EDAs) have been analyzed in multi-valued domains, see, e.g., \cite{BenJedidiaDK24,AdakW24
}. Due to the very different algorithmic concept of an EDA, these works are hardly comparable to ours.


Overall, these results demonstrate that runtime analysis in multi-valued decision spaces is feasible in the single-objective setting, but typically requires more involved techniques than in the binary case. 

\section{Preliminaries}

Throughout the paper, for all $a \le b$, we write $[a..b] := \{i \in \Z \mid a \le i \le b\}$ to denote the interval of integers between $a$ and $b$.

\subsection{Multi-Objective Optimization}

This paper considers multi-objective optimization problems defined on the search space $[0..r-1]^n$, where $r>2$. Such a problem is given by $f : [0..r-1]^n \rightarrow \mathbb{R}^m$. Our goal is to maximize the objectives. For two solutions $x,y \in [0..r-1]^n$, we say that $x$ \emph{weakly dominates} $y$ if $f_i(x) \ge f_i(y)$ holds for all $i \in [1..m]$, denoted as $x\succeq y$. If, in addition, there exists at least one objective $j$ such that $f_j(x) > f_j(y)$, then $x$ is said to \emph{dominate} $y$, denoted as $x\succ y$. A solution is called \emph{Pareto optimal} if it is not dominated by any other solution in the search space. The set of all Pareto-optimal solutions is referred to as the \emph{Pareto set}. The \emph{Pareto front} consists of the objective vectors corresponding to these solutions. As common in the evolutionary computation theory community~\cite{
ZhouYQ19,DoerrN20}, we measure the runtime by the number of function evaluations for the population to cover the entire Pareto front. We often use asymptotic notation (big-Oh notation, Landau symbols). If so, then the underlying variable is the problem size~$n$. We allow other quantities to be functionally dependent on~$n$, for example, the number $r$ of different values taken by the decision variables. This is natural, for example, in rooted tree problems in graphs on $n$ vertices, we have $n-1$ decision variables (describing the predecessor of a non-root node on the path to the source), and each of them can take any other node as value (so there are $n-1$ different values).

\subsection{The Multi-Valued Algorithm}

The \emph{simple evolutionary multi-objective optimizer (SEMO)} was the algorithm first used in a mathematical runtime analysis of an MOEA \cite{LaumannsTZWD02}. It is still the algorithm best understood from a theoretical perspective. Starting from  a single random solution, the algorithm repeatedly selects a random parent from the population, generates an offspring via mutation, and keeps as next population a largest pair-wise non-dominated subset of the old population and the offspring, giving preference to the offspring in case of ties. 

In the classic binary setting, the SEMO employs one-bit mutation, which flips a single bit chosen uniformly at random from the $n$ bit positions. In the multi-valued setting considered in this work, we employ a \emph{unit-strength local mutation} operator, which has, to the best of our knowledge, first been proposed and argued for in the single-objective work \cite{DoerrDK18}. This operator, as one of several alternatives, was also studied in the recent work on multi-objective optimization for unbounded integer-valued problems \cite{DoerrKR25}. When applying this mutation operator to an individual $x \in [0..r-1]^n$, one position $i \in [1..n]$ is chosen uniformly at random, and the value $x_i$ of this position is increased or decreased by $1$ with equal probability. A mutation that would lead to a value outside the allowed interval $[0..r-1]$ is considered infeasible and discarded (that is, the parent of the mutation operation is returned as offspring). The pseudocode of the algorithm is given in Algorithm~\ref{alg:semo}.

\begin{algorithm2e}[!h]
\caption{The $r$-valued SEMO with unit-strength local mutation to maximize $f:[0..r-1]^n \to \mathbb{R}^m$.}
\label{alg:semo}
Initialize $x \in [0..r-1]^n$ uniformly at random and set $P_0 \leftarrow \{x\}$\;
\For{$t = 0,1,2,\dots$}{
    Select one individual $x$ uniformly at random from $P_t$\;
    \label{line:parent_selection}
    Generate $y$ via choosing one entry of $x$ uniformly at random and modifying it by $\pm 1$ with equal probability\;
    \If{there is no $z \in P_t$ such that $z \succ y$}
    {\label{line:dominates}
        $P_{t+1} \leftarrow \{z \in P_t \mid y \nsucceq z\} \cup \{y\}$\;
    }
    \Else{
        $P_{t+1} \leftarrow P_t$\;
    }
}
\end{algorithm2e}

\subsection{The Multi-Valued Benchmark}

The \oneminmax problem is the best studied benchmark in the theory of MOEAs. Its two objectives are the number of zeros in the bit string and the number of ones. This is equivalent to saying that the objectives are the sum of the bit values and the sum of the inverses of the bit values. We therefore extend the binary \oneminmax benchmark to the multi-valued search space $[0..r-1]^n$ by letting the first objective be the sum of the values of the entries of $x$ and the second objective be the sum of the inverses. Since in the single-objective multi-valued setting, Adak and Witt~\shortcite{AdakW24} called this first objective G-\onemax, we refer to our multi-valued extension of the \oneminmax benchmark as the \emph{generalized \oneminmax} problem, denoted by \goneminmax. 

\begin{definition}
\label{def:gomm}
    Let $n \in \N$ and $r\in \N_{\geq 2}$. The \goneminmax function $f(x)=(f_1(x),f_2(x)):[0..r-1]^n\xrightarrow{} \R^2$ is defined by
    \begin{align*}
        &f_1(x) = \sum_{i=1}^n x_i, &f_2(x)=\sum_{i=1}^n (r-1-x_i).
    \end{align*}
\end{definition}

The following lemma characterizes the Pareto set and the Pareto front of \goneminmax. As in the binary case, every solution in the search space is Pareto optimal, but the size of the Pareto front also increases with~$r$. 

\begin{lemma}
\label{lem:pareto_front}
    The Pareto set of \goneminmax is $S^*=[0..r-1]^n$. The Pareto front is $F^*=\{(a,nr-n-a)\mid a\in[0..nr-n]\}$,
    whose size is $nr-n+1$.
\end{lemma}


{As a second benchmark in the experimental section, we use a multi-valued analogue of the classic \lotz benchmark. With very few solutions Pareto-optimal, it can be seen as a problem  complementary to \goneminmax. The first objective follows the idea of the $r$-valued \leadingones benchmark introduced in \cite{BenJedidiaDK24}, but refines it by also giving value-progress credit to the first position after the certified prefix. Hence a complete prefix of entries equal to $r-1$ contributes fully, the first subsequent position contributes its current value, and all later positions contribute nothing. Symmetrically, the second objective measures the same progress from the right towards value $0$. 
Hence, values lying strictly between the active leading position and the active trailing position are deliberately ignored: they are neither part of the prefix of $r-1$ values nor part of the suffix of $0$ values. This construction makes the Pareto-optimal solutions form a narrow monotone path, while preserving $n(r-1)+1$ Pareto-front points, as in \goneminmax.
\begin{definition}
\label{def:glotz}
    Let $n \in \N$ and $r\in \N_{\geq 2}$. The \glotz function $f(x)=(\glo(x),\gtz(x)):[0..r-1]^n\xrightarrow{} \R^2$ is defined by
    \begin{align*}
        &\glo(x) = \sum_{i=1}^n \left( \prod_{j=1}^{i-1} \mathbf{1}[x_j=r-1]\right) x_i, \\
        &\gtz(x)=\sum_{i=1}^n \left( \prod_{j=i+1}^{n} \mathbf{1}[x_j=0]\right) (r-1-x_i).
    \end{align*}
\end{definition}
For $r=2$, this definition coincides with the usual binary \lotz benchmark. The following lemma characterizes the Pareto-optimal solutions and Pareto front of this benchmark.
\begin{lemma}
\label{lem:glotz_pareto_front}
    For each $k\in[0..n(r-1)]$, write $q=\lfloor k/(r-1)\rfloor$ and $a=k-q(r-1)$. Let
    \[
    x^{(k)}=
    \begin{cases}
    (\underbrace{r-1,\ldots,r-1}_{q
    },a,\underbrace{0,\ldots,0}_{n-q-1
    }),& \text{if } q<n,\\
    (r-1,\ldots,r-1),& \text{if } q=n.
    \end{cases}
    \]
    The Pareto set of \glotz is the set of solutions $S^*_{\glotz}=\{x^{(k)}\mid k\in[0..n(r-1)]\}$ and its Pareto front is
    $F^*_{\glotz}=\{(k,n(r-1)-k)\mid k\in[0..n(r-1)]\}$. Hence both have size $n(r-1)+1$.
\end{lemma}
}


\section{Runtime Analysis}

\subsection{Mathematical Tools}

Before presenting the mathematical runtime analyses, we first introduce several tools that will be used later. To bound the objective value of the initial solution, we employ the following Chernoff bound.

\begin{theorem} [\cite{Chernoff52}]
\label{the:chernoff}
    Let $X_1,\dots,X_n$ be independent random variables taking values in $[0,r-1]$. Let $X=\sum_{i=1}^nX_i$. Let $\delta \in [0,1]$. Then 
    \begin{align*}
        \Pr\left[X\leq(1-\delta)\mathbb{E}[X]\right]\leq \exp\left(-\frac{\delta^2E[X]}{2(r-1)}\right).
    \end{align*}
\end{theorem}

The following tail bound for sums of independent geometric random variables provides strong guarantees for analyzing population dynamics in our lower-bound analysis.

\begin{theorem} [\cite{Witt14}]
\label{the:witt14}
    Let $k\in\mathbb{N}_{\geq1}$, and let $\{D_i\}_{i\in[k]}$ be independent geometric random variables with respective positive success probabilities $(p_i)_{i\in[k]}$. Let $T^* := \sum_{i\in[k]} D_i$, $s:=\sum_{i\in[k]}\frac{1}{p_i^2}$, and $\pmin:=\min\{p_i\mid i\in[k]\}$. Then for all $\lambda \in \mathbb{R}_{\geq0}$, we have 
    \begin{align*}
        &\Pr\left[T^*\geq\mathbb{E}[T^*]+\lambda\right]\leq \exp\left({-\frac{1}{4}\min\left\{\frac{\lambda^2}{s},\lambda \pmin \right\}}\right) ,\\
        &\Pr\left[T^*\leq\mathbb{E}[T^*]-\lambda\right]\leq \exp\left({-\frac{\lambda^2}{2s}}\right).
    \end{align*}
\end{theorem}

Drift analysis is a fundamental tool in the runtime analysis of randomized search heuristics \cite{HeY01}. The multiplicative drift theorem \cite{DoerrJW12algo} applies when the expected progress is proportional to the current distance to the target. 

\begin{theorem}[\cite{DoerrJW12algo}]
\label{the:drift analysis}
    Let $X^{(0)}, X^{(1)}, \ldots$ be a random process taking values in
    $S := \{0\} \cup [s_{\min}, \infty) \subseteq \mathbb{R}$.
    Assume that $X^{(0)} = s_0$ with probability one.
    Assume that there is a $\delta > 0$ such that for all $t \ge 0$ and all
    $s \in S$ with $\Pr[X^{(t)} = s] > 0$ we have
    \[
    \mathbb{E} [X^{(t+1)} \mid X^{(t)} = s] \le (1-\delta)s.
    \]
    Then $T := \min\{t \ge 0 \mid X^{(t)} = 0\}$
    satisfies
    \[
    \mathbb{E}[T] \le \frac{\ln(s_0/s_{\min}) + 1}{\delta}.
    \]
\end{theorem}

\subsection{Upper Bound for the Runtime of the SEMO on \goneminmax}~\label{ssec:upper}

We first bound the size of a pair-wise non-dominated set w.r.t. \goneminmax. 

\begin{lemma}
\label{lem:population_size}
    Let $P \subseteq [0..r-1]^n$ be any set of solutions that none weakly dominates the other w.r.t. \goneminmax. 
    Then $|P| \le nr-n+1$.
\end{lemma}


The following theorem provides an upper bound on the expected number of iterations until Algorithm~\ref{alg:semo} covers the whole Pareto front of \goneminmax. In contrast to the binary \oneminmax problem, where the objective value directly determines the number of flippable bits, solutions with the same objective value in our setting have a much more diverse structure. Hence a key part of the proof is to estimate the number of entries of a solution that can be mutated to derive an improvement in one objective. 
{More precisely, once the population contains an objective value, a neighboring value can be created by selecting a corresponding individual and mutating one suitable non-boundary coordinate in the correct direction. Summing the resulting waiting times over all missing front values yields the upper bound.}

\begin{theorem}
\label{the:upper_bound_SEMO}
    Consider Algorithm~\ref{alg:semo} optimizing \goneminmax. Then the expected number of iterations until the population covers the whole Pareto front is $O(n^2r^2\log n)$.
\end{theorem}

\subsection{Lower Bound for the SEMO Algorithm}\label{ssec:lower}

To derive a lower bound on the expected runtime of the SEMO, a technical difficulty is controlling the population size before the Pareto front is fully covered. Following the proof strategy introduced in~\cite{DoerrKO25}, we consider a \emph{delayed} SEMO that simplifies the analysis in this phase. The delayed SEMO is identical to the original SEMO (Algorithm~\ref{alg:semo}) except for the parent selection in line~\ref{line:parent_selection}, where it proceeds as follows. In each iteration, a value $ i \in [0..nr-n] $ is chosen uniformly at random. If there is no individual $ z \in P $ with $ f_1(z)=i $, the iteration is skipped; otherwise, the unique individual $x$ with $ f_1(x)=i $ is selected and the algorithm proceeds exactly as in line~4 of Algorithm~\ref{alg:semo}. This modification may introduce iterations that do not change the state of the  algorithm, but it has no influence otherwise. Hence it can only slow down the process. In particular, structural statements like the shape of the population at a particular stopping time are not affected. The advantage of working with the delayed SEMO is the more regular way parents are selected: In each iteration, the probability to choose a parent with some existing objective value is exactly the reciprocal of the size of the Pareto front. 
Since we will introduce another variant of the SEMO algorithm later on, we present the pseudocode of the delayed SEMO algorithm here to clearly distinguish and refer to the different algorithms. See Algorithm~\ref{alg:modified_semo}. 

\begin{algorithm2e}[!h]
\caption{The delayed $r$-valued SEMO with unit-strength local mutation to maximize $f:[0..r-1]^n \to \mathbb{R}^m$.}
\label{alg:modified_semo}
Initialize $x \in [0..r-1]^n$ uniformly at random and set $P_0 \leftarrow \{x\}$\;
\For{$t = 0,1,2,\dots$}{
    Choose $ i \in [0..nr-n] $ uniformly at random\;
    \If{ there is no $x \in P_t$ such that $f_1(x)=i$ }
    {
        Continue;
    }
    \Else{
        Select one individual $x$ such that $f_1(x)=i$;\label{stp:act}
    }
    Generate $y$ via choosing one entry of $x$ uniformly at random and modifying it by $\pm 1$ with equal probability\;
    \If{there is no $z \in P_t$ such that $z \succ y$}
    {\label{line:dominates}
        $P_{t+1} \leftarrow \{z \in P_t \mid y \nsucceq z\} \cup \{y\}$\;
    }
    \Else{
        $P_{t+1} \leftarrow P_t$\;
    }
}
\end{algorithm2e}

We use Algorithm~\ref{alg:modified_semo} for analysis until the size of the population has grown to linear order $ \Theta(n(r-1)) $ and then switch back to the original SEMO to then analyze the time required to cover the Pareto front.
In Lemma~\ref{lem:low_SEMO_1}, we show that the size of the population grows to $\Theta(n(r-1))$ within $O(n^2(r-1)^2)$ iterations with high probability. The key idea of the proof is that as long as the population is still small, there always exists a missing neighboring objective value, which can be generated with a constant probability conditioned on selecting a suitable individual. Repeatedly filling such gaps, the size of the population increases steadily until it reaches the desired order.
This result provides a lower bound on the population size that will be crucial in the final lower-bound analysis.

\begin{lemma}
\label{lem:low_SEMO_1}
    Consider Algorithm~\ref{alg:modified_semo} optimizing \goneminmax. Then with probability $1-\exp(-\Theta(n))$, the population contains at least $n(r-1)/130$ individuals after at most $8n^2(r-1)^2/129$ iterations.
\end{lemma}

{For a population $P$, we call the two extreme objective vectors $(0,n(r-1))$ and $(n(r-1),0)$ the \emph{Pareto borders}. We measure the distance of $P$ to these borders by 
\[d_{PF}(P):=\min\Big\{\min_{z\in P}f_1(z),\min_{z\in P}f_2(z)\Big\}.\]
Hence $d_{PF}(P)=0$ is necessary for covering the full Pareto front.}
The following lemma shows that, with high probability, the distance of the population to the Pareto borders is still considerable within $n^2(r-1)^2/16$ iterations. The proof relies on a concentration argument for the random initialization and on bounding the maximum drift towards the Pareto border with unit-strength local mutation.

\begin{lemma}
\label{lem:low_SEMO_2}
    Consider Algorithm~\ref{alg:modified_semo} optimizing \goneminmax. Then with probability at least $1-\exp (-\Theta(n))$, in all iterations $t \in [0..n^2(r-1)^2/16]$, the distance of the population to the Pareto borders is at least $n(r-1)/8$.
\end{lemma}

As mentioned above, we switch back to the analysis of the original SEMO algorithm. Conditioning on a sufficiently large population but a considerable distance to the Pareto borders, we analyze the expected time required to reach the Pareto borders. {Since reaching one of the borders is necessary for covering the full Pareto front, we follow the progress of the smaller attained objective value towards the nearest border. The proof splits this progress into a far-from-border phase and a close-to-border phase. In both phases, we upper-bound the probability of decreasing the remaining distance by one, which yields lower bounds on the required waiting times.}

\begin{lemma}
\label{lem:low_SEMO_3}
    Consider Algorithm~\ref{alg:semo} optimizing \goneminmax, starting with a population size of at least $\frac{n(r-1)}{130}$ and the distance to the Pareto borders of at least $n(r-1)/8$. Then, the expected number of iterations until the population covers the whole Pareto front is $\Omega(n^2r(r+\log n))$.
\end{lemma}

Combining the above three lemmas, we derive the following theorem. 

\begin{theorem}
\label{the:lower_bound_SEMO}
    Consider Algorithm~\ref{alg:semo} optimizing \goneminmax. Then
    the expected number of iterations until the population covers the whole Pareto front is $\Omega(n^2r(r+\log n))$.
\end{theorem}


\subsection{A Tighter Bound for a Variant of the SEMO Algorithm}\label{ssec:tight}

We note that the two bounds proven so far are not absolutely tight, but differ by a factor of $O(\log n)$. Due to the complicated population dynamics in a run of an MOEA, tight runtime bounds are generally hard to obtain, and in fact, the only ones in this field we are aware of are those in  \cite{LaumannsTZWD02,DoerrZ21aaai,BossekS24,DoerrQ23LB,DoerrKO25,Opris26aaai}. To try to shed more light on the true runtime of the multi-valued SEMO, we now consider a mild modification of the SEMO. It differs from the original algorithm only in the tie-breaking rule in line~\ref{line:dominates} of Algorithm~\ref{alg:semo}, where now we accept only offspring that are not even weakly dominated. In other words, if a solution with some objective value is already present in the population, then it is not replaced by a later solution with equal objective value. We note that this is exactly the algorithm proposed in the original work introducing the SEMO \cite{LaumannsTZWD02}, and only later the research community gradually adopted the version presented in Algorithm~\ref{alg:semo}. To avoid any ambiguity, the pseudocode of the algorithm regarded in this subsection can be found in Algorithm~\ref{alg:variant_of_SEMO}.

\begin{algorithm2e}[!h]
\caption{The variant of the $r$-valued SEMO that does not accept solutions which are weakly dominated. As before, we use unit-strength local mutation and maximize $f:[0..r-1]^n \to \mathbb{R}^m$.}
\label{alg:variant_of_SEMO}
Initialize $x \in [0..r-1]^n$ uniformly at random and set $P_0 \leftarrow \{x\}$\;
\For{$t = 0,1,2,\dots$}{
    Select one individual $x$ uniformly at random from $P_t$\;
    Generate $y$ via choosing one entry of $x$ uniformly at random and modifying it by $\pm 1$ with equal probability\;
    \If{there is no $z \in P_t$ such that $z \succeq y$}
    {\label{line:dominates}
        $P_{t+1} \leftarrow \{z \in P_t \mid y \nsucceq z\} \cup \{y\}$\;
    }
    \Else{
        $P_{t+1} \leftarrow P_t$\;
    }
}
\end{algorithm2e}

For Algorithm~\ref{alg:variant_of_SEMO}, we can prove an upper bound that agrees with the previously shown lower bound (which is valid here as well). 
{The proof is based on drift analysis and exploits the modified tie-breaking rule, which rejects offspring with an already represented objective vector. This removes neutral replacements and allows us to define a potential function that decreases only when real progress towards one of the two extreme Pareto-optimal solutions, $0^n$ and $(r-1)^n$, is made. A multiplicative-drift argument then gives the desired upper bound.}
Since we do not believe that the small difference of accepting or not solutions with equal objective value changes the asymptotic runtime (which will be supported by our experiments), we see this result as a strong indication for the true runtime of the $r$-valued SEMO.

\begin{theorem}
\label{the:Theta_bound}
    Consider optimizing \goneminmax via Algorithm~\ref{alg:semo} with the modification that it accepts only strictly better solutions. Then the expected number of iterations until the population covers the whole Pareto front is $\Theta(n^2r(r+\log n))$.
\end{theorem}

\section{Experiments}

Above, we used Algorithm~\ref{alg:variant_of_SEMO} as a model for the SEMO which we could analyze well enough to obtain asymptotically tight bounds for the runtime. Our strong belief was that the small difference in the selection procedure has little influence on the performance. We now support this belief with {experiments on both \goneminmax and \glotz}.


For \goneminmax, we conduct two sets of experiments. In the first set, we fix $r = 4$ and consider the problem sizes $n \in \{20,40,60,80,100\}$; in the second, we fix $n = 100$ and consider the $r$-values $r \in \{2,3,4,5\}$. For each setting, we conduct 100 independent runs of both algorithms and record the first iteration after which the population covers the full Pareto front.
{We repeat the same experimental setup for \glotz, which has the same front size $n(r-1)+1$ but only one Pareto-optimal solution for each objective vector. This allows us to test whether our empirical conclusions are robust beyond the special \goneminmax case, where all solutions are Pareto optimal.}


Figures~\ref{fig:r_4} and~\ref{fig:n_100} display the observed runtimes (averages and standard deviations). The results indicate not greater differences between the original SEMO algorithm (Algorithm~\ref{alg:semo}) and its variant Algorithm~\ref{alg:variant_of_SEMO}; possibly the original SEMO performs slightly better.
{Figures~\ref{fig:glotz_r_4} and~\ref{fig:glotz_n_100} show the corresponding \glotz results. Again Algorithm~\ref{alg:variant_of_SEMO} behaves similarly to the original SEMO algorithm.}
Hence our experiments support our belief that the small variation in the selection does not drastically change the algorithm behavior, and more specifically, that our improved upper bound on the runtime shows for  Algorithm~\ref{alg:variant_of_SEMO} is likely to hold also for the SEMO.

\begin{figure}[!h]
    \centering
    \includegraphics[width=1\linewidth]{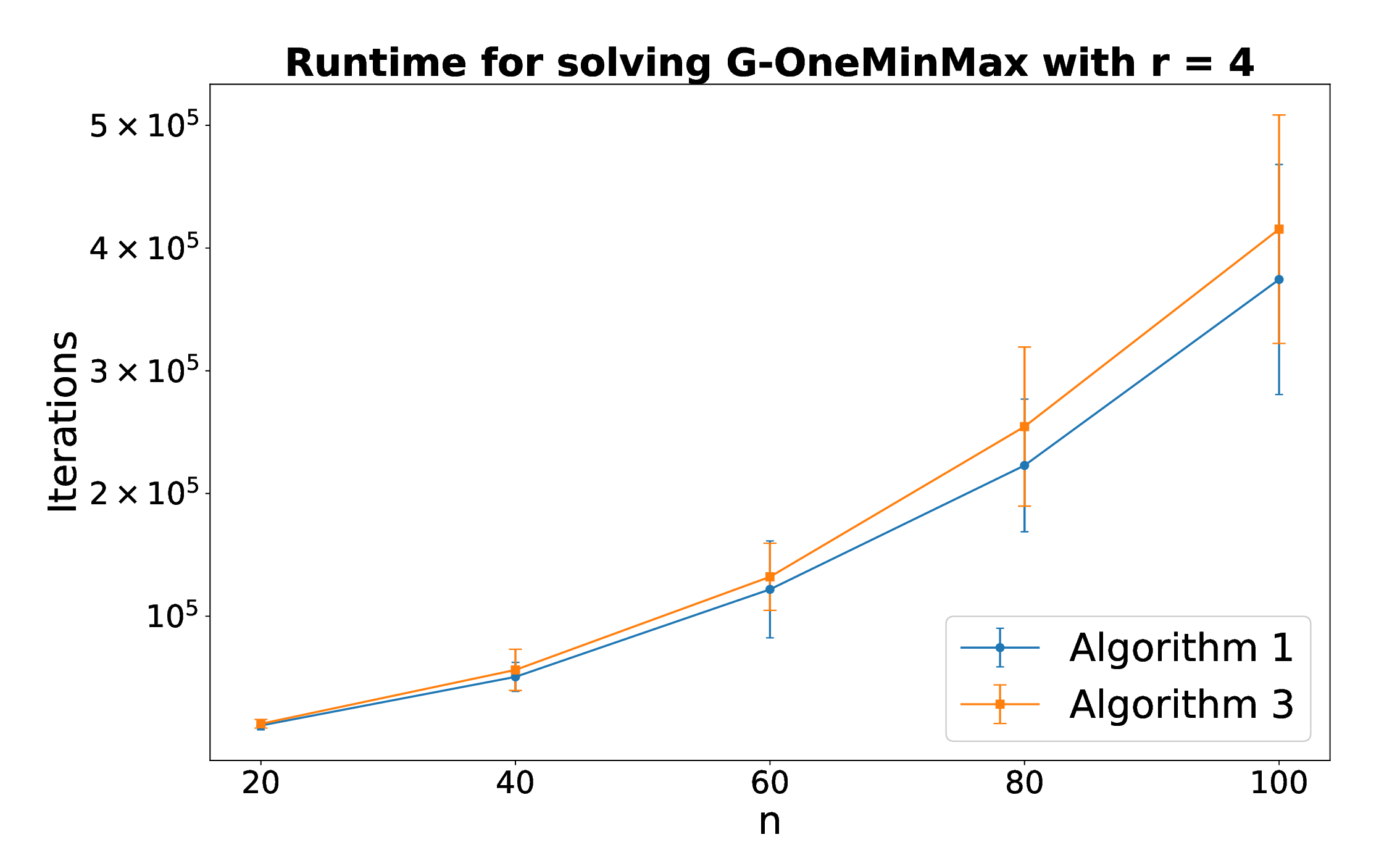}
    \caption{The mean (with standard deviations) number of iterations of Algorithms~\ref{alg:semo} and~\ref{alg:variant_of_SEMO} for solving \goneminmax with $n\in \{20,40,60,80,100\}$ and $r=4$ in 100 independent runs.}
    \label{fig:r_4}
\end{figure}
\begin{figure}[!h]
    \centering
    \includegraphics[width=1\linewidth]{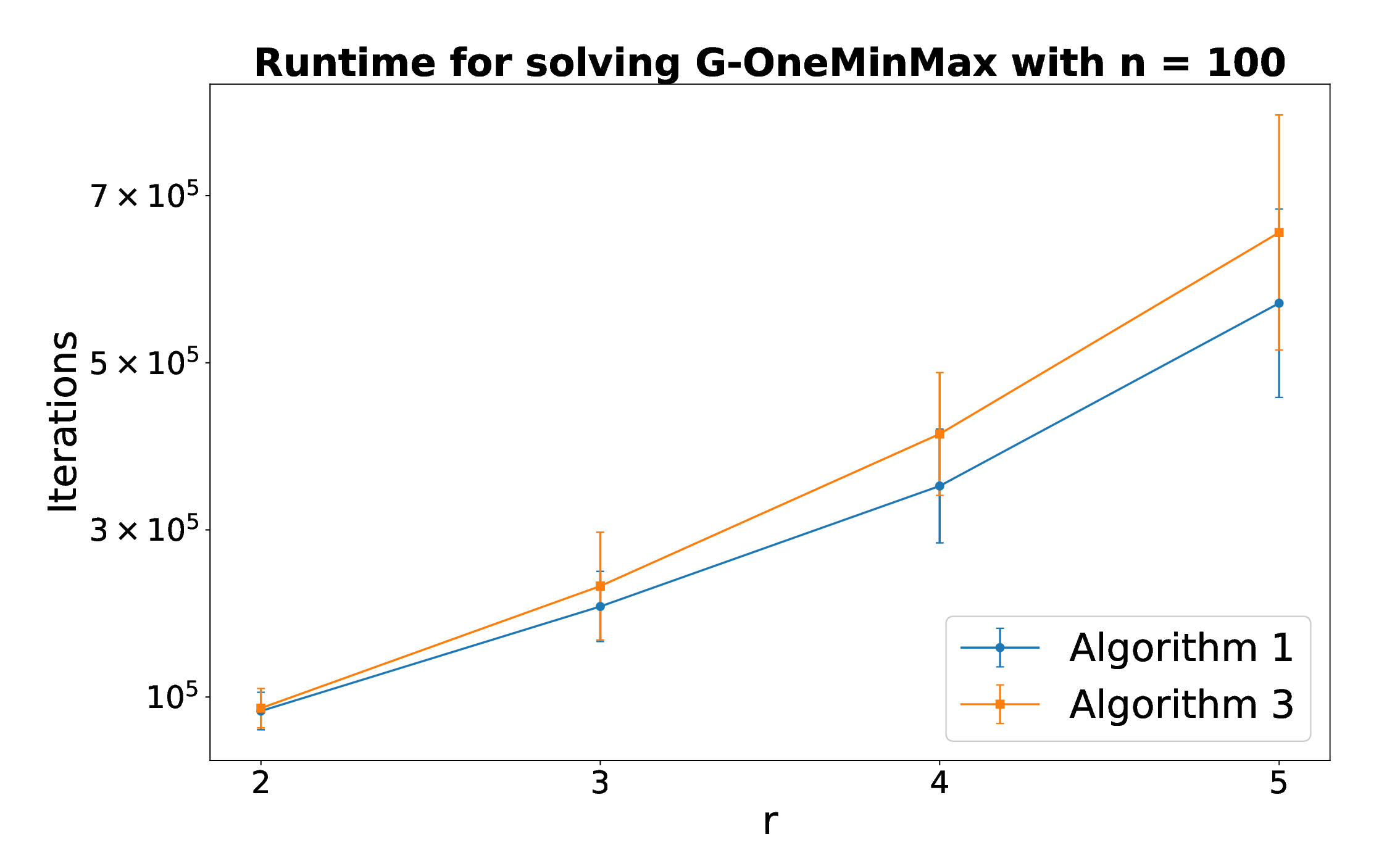}
    \caption{The mean (with standard deviations) number of iterations of Algorithms~\ref{alg:semo} and~\ref{alg:variant_of_SEMO} for solving \goneminmax with $n=100$ and $r\in [2..5]$ in 100 independent runs.}
    \label{fig:n_100}
\end{figure}

\begin{figure}[!h]
    \centering
    \includegraphics[width=1\linewidth]{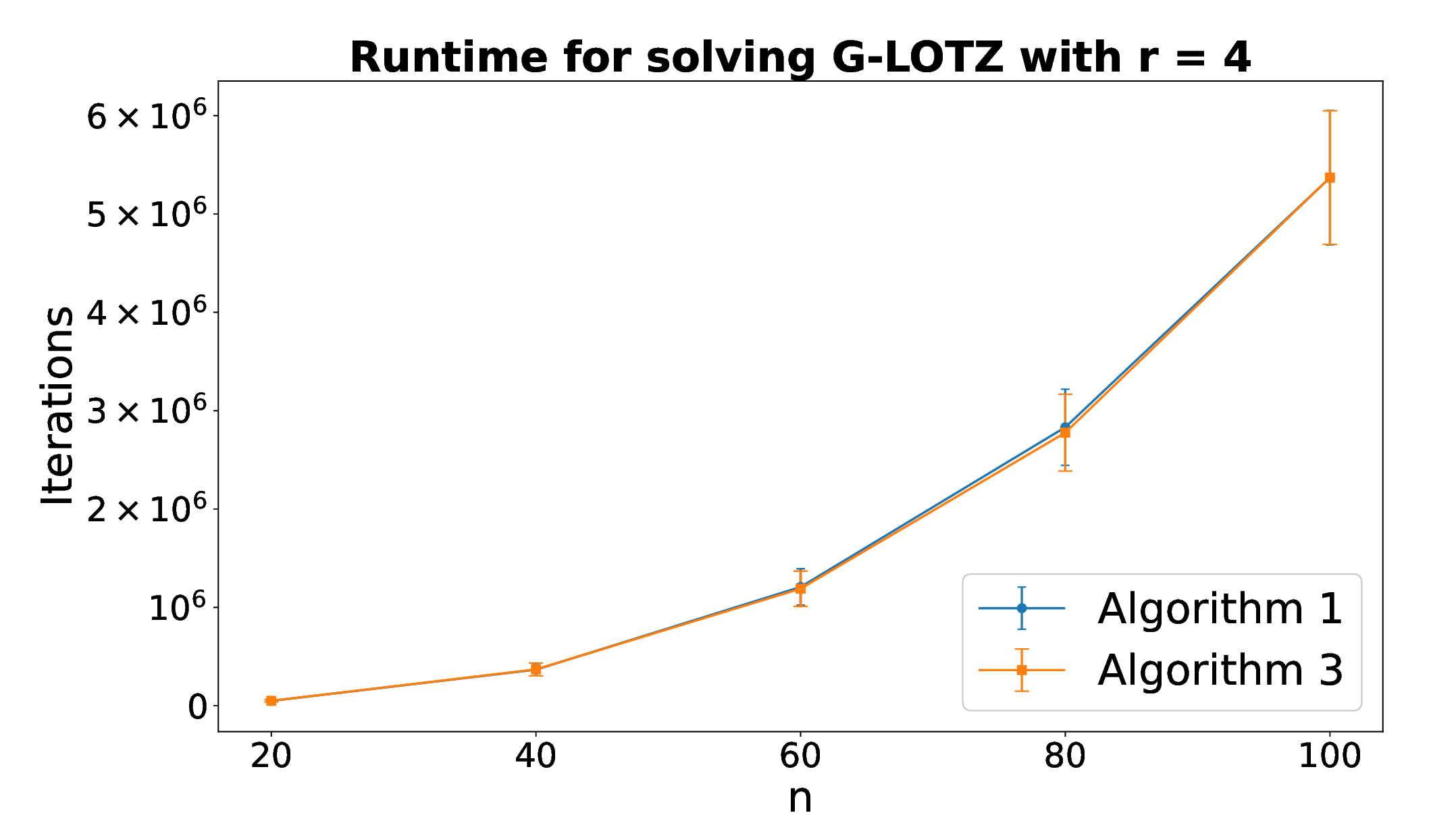}
    \caption{The mean (with standard deviations) number of iterations of Algorithms~\ref{alg:semo} and~\ref{alg:variant_of_SEMO} for solving \glotz with $n\in \{20,40,60,80,100\}$ and $r=4$ in 100 independent runs.}
    \label{fig:glotz_r_4}
\end{figure}
\begin{figure}[!h]
    \centering
    \includegraphics[width=1\linewidth]{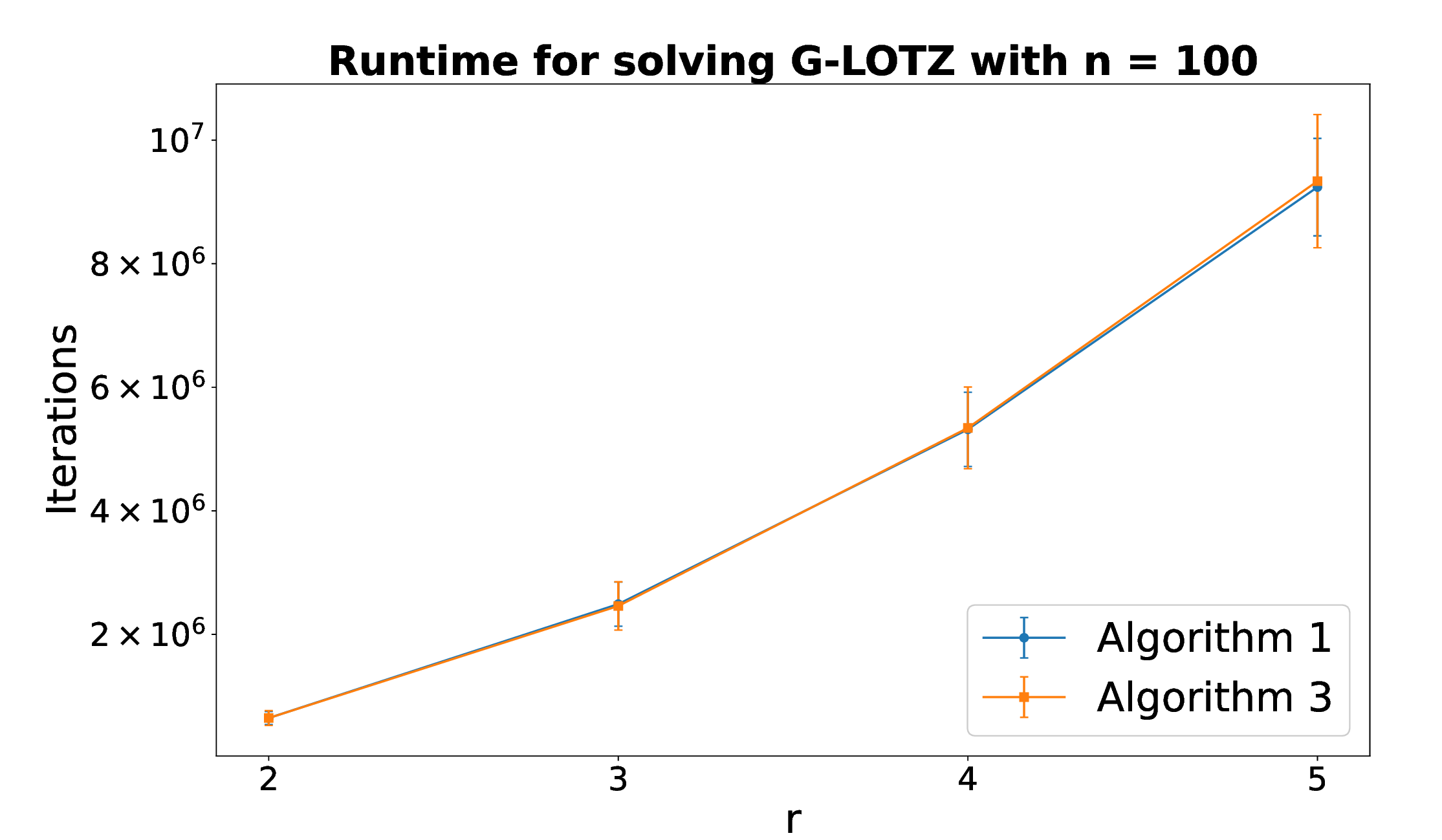}
    \caption{The mean (with standard deviations) number of iterations of Algorithms~\ref{alg:semo} and~\ref{alg:variant_of_SEMO} for solving \glotz with $n=100$ and $r\in [2..5]$ in 100 independent runs.}
    \label{fig:glotz_n_100}
\end{figure}

\section{Conclusion}

This paper conducted the first mathematical runtime analyses of MOEAs for multi-valued decision variables. We extended the classic \oneminmax problem and the SEMO multi-objective algorithm from the binary setting to finite multi-valued domains. We proved an upper bound of  $O(n^2r^2\log n)$ on the expected number of iterations required to compute the entire Pareto front. We also proved a near-tight lower bound of $\Omega(n^2r(r+\log n))$. Both agree with the known $\Theta(n^2 \log n)$ runtime for the binary case. 

These results suggest that existing MOEAs, when equipped with natural extensions of the mutation operator to multi-valued domains, perform well on $r$-valued problems. The moderate increase in runtime as $r$ grows is to be expected, among others, because the size of the Pareto front of our optimization problem grows linearly with~$r$. 

In contrast, the current mathematical analysis techniques appear insufficient to completely understand such multi-valued optimization processes. At present, we can determine the true asymptotic order of magnitude of the runtime only for the variant of the SEMO that accepts only strictly dominating solutions. Here we prove a runtime of $\Theta(n^2 r(r+ \log n))$, showing a linear increase of the runtime with $r$ when $r$ is at most logarithmic in~$n$. Both from our understanding of the two algorithms and our experiments, we conjecture that both algorithms have the same asymptotic runtime on the \goneminmax benchmarks, which then would be $\Theta(n^2 r(r+ \log n))$. Proving such a result formally, however, appears to be challenging and might need considerably new analysis methods. 

\subsection*{Acknowledgments}
This work was supported by Guangdong Basic and Applied Basic Research Foundation (Grant No. 2025A1515011936), National Natural Science Foundation of China (Grant No. 62306086), and Xinjiang Tianshan Innovative Research Team (2025D14009). 
This research benefited from the support of the FMJH Program PGMO.

}

\bibliographystyle{named-abbrv}
\bibliography{ich_master,alles_ea_master,rest}
\newpage
\appendix
\onecolumn
\leftlinenumbers
\setcounter{linenumber}{0}
\theoremstyle{NoParens}
\newtheorem*{theoremNP}{Theorem}
\newtheorem*{lemmaNP}{Lemma}
\newtheorem*{corollaryNP}{Corollary}
\newtheorem*{definitionNP}{Definition}
\section*{Supplementary Material for Paper 2110 - ``First Mathematical Runtime Analyses of Multi-Objective Evolutionary Algorithms for Multi-Valued Decision Variables"}
This document provides the proofs omitted from the main paper due to space limitations. For convenience, all lemmas and theorems are restated using the same numbering as in the main paper, and the line number is reset to 1 at the beginning.
\section{Omitted Proofs}
\begin{lemmaNP}[\ref{lem:pareto_front}]
    The Pareto set of \goneminmax is $S^*=[0..r-1]^n$. The Pareto front is 
    \begin{align*}
        F^*=\{(a,nr-n-a)\}\mid a\in[0..nr-n],
    \end{align*}
    whose size is $nr-n+1$.
\end{lemmaNP}
\begin{proof}[Proof of Lemma~\ref{lem:pareto_front}]
    We first show that every solution in $[0..r-1]^n$ is Pareto optimal. Let $x,y\in[0..r-1]^n$. By Definition~\ref{def:gomm}, we have $f_1(x)=\sum_{i=1}^n x_i$ and $ f_2(x)= n(r-1)-f_1(x).$ Hence, $f(x)=(f_1(x),\,n(r-1)-f_1(x))$ lies on the line $f_1+f_2=n(r-1)$. If $f_1(y)\neq f_1(x)$, then one objective strictly increases while the other strictly decreases, implying that $y$ cannot dominate $x$. If $f_1(y)=f_1(x)$, then also $f_2(y)=f_2(x)$ and hence $y\nsucc x$. Therefore, every solution in the search space is non-dominated. The Pareto set equals the whole search space, that is, $S^*=[0..r-1]^n$. Since $f_1(x)=\sum_{i=1}^n x_i$ ranges over all integers in $[0..n(r-1)]$, the Pareto front is $F^*=\{(a,\,n(r-1)-a)\mid a\in[0..n(r-1)]\},$, whose size is $nr-n+1$.
\end{proof}

\begin{lemmaNP}[\ref{lem:glotz_pareto_front}]
    For each $k\in[0..n(r-1)]$, write $q=\lfloor k/(r-1)\rfloor$ and $a=k-q(r-1)$. Let
    \[
    x^{(k)}=
    \begin{cases}
    (\underbrace{r-1,\ldots,r-1}_{q},a,\underbrace{0,\ldots,0}_{n-q-1}),& q<n,\\
    (r-1,\ldots,r-1),& q=n.
    \end{cases}
    \]
    The Pareto set of \glotz is $S^*_{\glotz}=\{x^{(k)}\mid k\in[0..n(r-1)]\}$ and its Pareto front is
    $F^*_{\glotz}=\{(k,n(r-1)-k)\mid k\in[0..n(r-1)]\}$. Hence both have size $n(r-1)+1$.
\end{lemmaNP}

\begin{proof}[Proof of Lemma~\ref{lem:glotz_pareto_front}]
    Let $M:=r-1$. For a solution $x\in[0..M]^n$, define $p:=\min\{i\in[1..n]\mid x_i<M\}$ with $p=n+1$ if no such index exists. Similarly, define $s:=\max\{i\in[1..n]\mid x_i>0\}$ with $s=0$ if no such index exists. By the definition of \glo, if $p\le n$, then $\glo(x)=(p-1)M+x_p$,
    and if $p=n+1$, then $\glo(x)=nM$. Analogously, if $s\ge1$, then $\gtz(x)=(n-s)M+(M-x_s)$, and if $s=0$, then $\gtz(x)=nM$.
    
    We first prove that every solution satisfies $\glo(x)+\gtz(x)\le nM$. The boundary cases are immediate. if $p=n+1$, then $x=(M,\ldots,M)$ and $\glotz(x)=(nM,0)$; if $s=0$, then $x=(0,\ldots,0)$ and $\glotz(x)=(0,nM)$. It remains to consider the case $p\le n$ and $s\ge1$. If $p<s$, then
    \begin{align*}
        \glo(x)+\gtz(x)
        &=(p-1)M+x_p+(n-s)M+(M-x_s)\\
        &=nM-\left((s-p)M+x_s-x_p\right).
    \end{align*}
    Since $s>p$, $x_p\le M-1$, and $x_s\ge1$, the subtracted term is positive. Hence the sum is strictly smaller than $nM$. If $p=s$, then all entries before $p$ are equal to $M$, all entries after $p$ are equal to $0$, and $x_p\in[1..M-1]$. In this case $\glo(x)+\gtz(x)=nM$. Finally, if $p>s$, then no index can lie strictly between $s$ and $p$ with value different from both $0$ and $M$. Since all positions before $p$ are equal to $M$ and all positions after $s$ are equal to $0$, we have $p=s+1$. Hence $x$ consists of a prefix of $M$'s followed by a suffix of $0$'s. Again, $\glo(x)+\gtz(x)=nM$.
    
    The above discussion also characterizes the equality case. Equality holds exactly for solutions consisting of a prefix of $M$'s, followed by at most one intermediate position, followed by a suffix of $0$'s. Equivalently, these solutions are of the form
    \[
    (\underbrace{M,\ldots,M}_{q\text{ times}},a,\underbrace{0,\ldots,0}_{n-q-1\text{ times}})
    \]
    with $q<n$ and $a\in[0..M]$, together with the all-$M$ solution. These are precisely the points $x^{(k)}$ stated in the lemma, where $k=qM+a$. For such a point $x^{(k)}$, we have directly
    \[
    \glo(x^{(k)})=qM+a=k
    \]
    and
    \[
    \gtz(x^{(k)})=(n-q-1)M+(M-a)=nM-k
    \]
    when $q<n$; the case $q=n$ gives $(nM,0)$ as well. Hence all stated solutions have objective vectors $(k,nM-k), k\in[0..nM]$.
    
    It remains to prove that these and only these solutions are Pareto optimal. Since every solution satisfies $\glo(x)+\gtz(x)\le nM$, no solution with objective vector $(k,nM-k)$ can be dominated: increasing one objective beyond its value would force the other objective below its value. Hence all $x^{(k)}$ are Pareto optimal. Conversely, if $x$ is not one of these solutions, then $\glo(x)+\gtz(x)<nM$.
    Let $u:=\glo(x)$. The front point $(u,nM-u)$ weakly dominates $(\glo(x),\gtz(x))$ and is strictly better in the second objective. Hence $x$ is not Pareto optimal. Therefore, the Pareto set is exactly
    \[
    S^*_{\glotz}=\{x^{(k)}\mid k\in[0..n(r-1)]\},
    \]
    and the Pareto front is exactly
    \[
    F^*_{\glotz}=\{(k,n(r-1)-k)\mid k\in[0..n(r-1)]\}.
    \]
    Both sets contain one element for each integer $k\in[0..n(r-1)]$, and thus both have size $n(r-1)+1$.
\end{proof}

\begin{lemmaNP}[\ref{lem:population_size}]
    Let $P \subseteq [0..r-1]^n$ be any set of solutions that none weakly dominates the other w.r.t. \goneminmax. 
    Then $|P| \le nr-n+1$.
\end{lemmaNP}
\begin{proof}[Proof of Lemma~\ref{lem:population_size}]
    From Lemma~\ref{lem:pareto_front}, we know that every solution $\in [0..r-1]^n$ is Pareto optimal for \goneminmax. Hence, the objective vectors corresponding to solutions in $P$ form a subset of the Pareto front. Since $P$ contains at most one solution for each objective vector, its size is bounded by the size of the Pareto front, that is, $nr-n+1$.
\end{proof}

\begin{theoremNP}[\ref{the:upper_bound_SEMO}]
    Consider Algorithm~\ref{alg:semo} optimizing \goneminmax. Then the expected number of iterations until the population covers the whole Pareto front is $O(n^2r^2\log n)$.
\end{theoremNP}

\begin{proof}[Proof of Theorem~\ref{the:upper_bound_SEMO}]
    Let $i$ denote the maximum $f_1$-value reached by the population, that is, $i := \max\{f_1(x)\mid x\in P_t\}.$ We first estimate the waiting time to increase this value. For the unit-strength local mutation, an offspring with $f_1=i+1$ can only be generated from an individual $x\in P_t$ with $f_1(x)=i$. From Lemma~\ref{lem:population_size}, we know that $|P_t|\leq nr-n+1$ and hence the probability to select one desired individual is at least $1/(nr-n+1)$. For any individual $x$ with $f_1(x)=i$, at most $\lfloor i/(r-1)\rfloor$ of its entries can have the value $r-1$. Consequently, at least $n-\lfloor i/(r-1)\rfloor$ entries have values strictly smaller than $r-1$. 
    Choosing one of these entries (with probability at least $(n-\lfloor i/(r-1)\rfloor)/n$) and choosing the $+1$ step (with probability $1/2$) yield an offspring $y$ with $f_1(y)=i+1$. Such an offspring is not dominated by any individual in $P_t$ and therefore enters into the next population. Let $p_i$ denote the probability to increase the current maximum $f_1$-value from $i$ to $i+1$ in one iteration. Then we have $p_i \ge \frac{1}{2}\cdot\frac{1}{nr-n+1}\cdot \frac{n-\lfloor i/(r-1)\rfloor}{n}. $
    
    Analogously, let $ j := \max\{f_2(x)\mid x\in P_t\}$ be the current maximum $f_2$-value in the population. An offspring with $f_2=j+1$ can only be generated from an individual $x\in P_t$ with $f_2(x)=j$, which is selected with probability at least $1/(nr-n+1)$. 
    For any such $x$, we have $\sum_{k=1}^n x_k = n(r-1)-j$, implying that at least $\lceil (n(r-1)-j)/(r-1)\rceil = \lceil n-j/(r-1)\rceil$ entries of $x$ have values strictly larger than $0$. 
    The probability to select one of these entries and then choose the $-1$ step is at least $\lceil n-j/(r-1)\rceil/(2n)$. The resulting offspring increases the $f_2$-value by one and will be accepted. Denoting by $q_j$ the probability that the current maximum $f_2$-value increases from $j$ to $j+1$ in one iteration, we obtain $ q_j \ge \frac{1}{nr-n+1}\cdot \frac{\lceil n-j/(r-1)\rceil}{2n}.$
    
    Since any reached Pareto front point will be maintained in all future generations, we know that the maximum $f_1$-value and the maximum $f_2$-value reached by $P_t$ will not decrease. Let $z$ be the initially generated individual. Then the expected number of iterations until the population covers the whole Pareto front is at most  
    \begin{align*}
        &\sum_{i=f_1(z)}^{nr-n-1} \frac{1}{p_i} + \sum_{j=f_2(z)}^{nr-n-1} \frac{1}{q_j}\\
        &\leq \sum_{i=f_1(z)}^{nr-n-1} \frac{2(nr-n+1)n}{n-\lfloor i/(r-1) \rfloor} + \sum_{j=f_2(z)}^{nr-n-1} \frac{2(nr-n+1)n}{\lceil  n-j/(r-1)  \rceil}\\
        &< \sum_{i=0}^{nr-n-1} \frac{2(nr-n+1) n}{n-\lfloor i/(r-1) \rfloor}+ \sum_{j=0}^{nr-n-1} \frac{2(nr-n+1)n}{\lceil  n-j/(r-1)  \rceil}\\
        &=4(nr-n+1)n(r-1)\sum_{i=1}^n \frac{1}{i} =O(n^2r^2\log n),
    \end{align*} 
    where the first equality follows from aggregating the equal terms: As $i$ and $j$ ranges from $0$ to $nr-n+1$, the values of $n-\lfloor i/(r-1) \rfloor$ and $\lceil  n-j/(r-1)  \rceil$ each take every value in $[1..n]$ exactly $r-1$ times.
\end{proof}

\begin{lemmaNP}[\ref{lem:low_SEMO_1}]
    Consider Algorithm~\ref{alg:modified_semo} optimizing \goneminmax. Then with probability $1-\exp(-\Theta(n))$, the population contains at least $n(r-1)/130$ individuals after at most $8n^2(r-1)^2/129$ iterations.
\end{lemmaNP}
\begin{proof}[Proof of Lemma~\ref{lem:low_SEMO_1}]
    Consider any time $t$ with $|P_t|<n(r-1)/130$. Let $A_t$ denote the set of $f_1$-values currently contained in the population, that is, $A_t:=\{f_1(z)\mid z\in P_t\}\subseteq [0..n(r-1)].$ We first show that, regardless of the current state of the population, there exists an $i\in A_t$ such that either $i+1\notin A_t$ or $i-1\notin A_t$. Then we calculate the probability to generate a solution such that its $f_1$-value equals to this missing value. 
    
    We consider the following two cases. For the first case, assume that there is an individual $x'\in P_t$ with $f_1(x')\le \lfloor 129n(r-1)/260\rfloor$. Then among the values $f_1(x'), f_1(x')+1,\dots,f_1(x')+\lfloor n(r-1)/130\rfloor$, there exists an $m\in[0..n(r-1)/130]$ such that $f_1(x')+m\notin A_t$. Let $m$ be minimal with this property and define $i:=f_1(x')+m-1$. We know that $i\in A_t$ but $i+1\notin A_t$. Let $x\in P_t$ be such an individual with $f_1(x)=i$. Then we have that $f_1(x)=i < \frac{n(r-1)}{130}+\Big\lfloor\frac{129n(r-1)}{260}\Big\rfloor,$ and no $y\in P_t$ with $f_1(y)=f_1(x)+1.$
    
    For the other case, all individuals $z\in P_t$ satisfy $f_1(z)>\lfloor 129n(r-1)/260\rfloor$. Let $i:=\min A_t$ and take the solution $x\in P_t$ with $f_1(x)=i$. We have $i-1\notin A_t$, i.e., there is no individual $y\in P_t$ with $f_1(y)=f_1(x)-1$.
    
    In both cases we obtained an individual $x\in P_t$ such that $f_1(x)+1$ or $f_1(x)-1$ is not represented in the population. We now bound from below the probability to generate an offspring in one iteration of Algorithm~\ref{alg:modified_semo} whose fitness fills this missing value. First, the modified parent selection chooses a value in $[0..n(r-1)]$ uniformly at random. Hence, the probability that it selects $f_1(x)$ is $1/(n(r-1)+1)$, and then $x$ is chosen as parent. 
    Second, in both cases we have $f_1(x)>\lfloor 129n(r-1)/260\rfloor-1$, which implies that at least $ \Big\lceil \frac{129n(r-1)/260}{r-1}\Big\rceil \ge \frac{129n}{260} $ entries is not at the relevant boundary and thus allow a unit-step change in the required direction while keeping all other entries unchanged. The probability to choose one of these entries is at least $(129n/260)/n$, and the probability to choose the correct $\pm 1$ direction is $1/2$. Therefore, in each iteration, the probability to create a missing neighboring $f_1$-value is at least
    \[
    p:=\frac{1}{n(r-1)+1}\cdot \frac{129}{260}\cdot \frac{1}{2}
    =\frac{129}{520\,(n(r-1)+1)}.
    \]

    Each successful creation of a previously missing neighboring value increases the population size by one.
    Consequently, $T$ is stochastically dominated by the sum of $\lceil n(r-1)/130-1\rceil$ independent geometric random variables $V_1,\dots,V_{\lceil n(r-1)/130\rceil-1}$ with success probability~$p$. Let $V:=\sum_{i=1}^{\lceil n(r-1)/130\rceil-1}V_i$. Then
    \begin{align*}
        \mathbb{E}[V]
        &=\bigl(\lceil n(r-1)/130\rceil-1\bigr)\cdot \tfrac{1}{p}\\
        &=\bigl(\lceil n(r-1)/130\rceil-1\bigr)\cdot \frac{520(n(r-1)+1)}{129}\\
        &\le \frac{4n(r-1)}{129}\cdot (n(r-1)+1)
        \le \frac{4.84n^2(r-1)^2}{129}.
    \end{align*}

    With Theorem~\ref{the:witt14},  we obtain that for $\lambda:=3.16n^2(r-1)^2/129$ and $s:=(\lceil n(r-1)/130\rceil-1)/p^2$, we have
    \begin{align*}
        \Pr[T\geq 8n^2(r-1)^2/129]&\leq \Pr[V\geq8n^2(r-1)^2/129]\\
        &\leq \Pr[V\geq \mathbb{E}(V)+3.16n^2(r-1)^2/129]\\
        &\leq \exp({-\tfrac{1}{4}\min\{\tfrac{\lambda^2}{s},\lambda p\}}) =\exp(-\Theta(n)). \qedhere
    \end{align*}
\end{proof}

\begin{lemmaNP}[\ref{lem:low_SEMO_2}]
    Consider Algorithm~\ref{alg:modified_semo} optimizing \goneminmax. Then with probability at least $1-\exp (-\Theta(n))$, in all iterations $t \in [0..n^2(r-1)^2/16]$, the distance of the population to the Pareto borders is at least $n(r-1)/8$.
\end{lemmaNP}
\begin{proof}[Proof of Lemma~\ref{lem:low_SEMO_2}]
    For each iteration $t$, define $X_t^\leftarrow := \min_{z\in P_t} f_1(z)$ and $X_t^\rightarrow := \min_{z\in P_t} f_2(z)$. Let $d_{PF}(P_t) =\min\{X_t^\leftarrow,X_t^\rightarrow\}$ denote the distance of the population to the Pareto borders. 
    Since the arguments for $X_t^\leftarrow$ and $X_t^\rightarrow$ are symmetric, we only analyze $X_t^\leftarrow$ here. 
    
    We first estimate the probability that $X_0^\leftarrow \geq n(r-1)/4$. Let $x$ be the random initial individual. We have $f_1(x)=\sum_{i=1}^n x_i$ and hence $E[f_1(x)]=\frac{n(r-1)}{2}$. Applying the Chernoff bound in Theorem~\ref{the:chernoff} with $\delta=1/2$, we obtain
    \[
    \Pr[f_1(x)\ge n(r-1)/4] \ge 1-\exp\left(-\tfrac{n}{16}\right).
    \]
    For the rest of the proof, we condition on this event and thus assume $X_0^\leftarrow\ge n(r-1)/4$.

    Due to unit-strength local mutation, the value $X_t^\leftarrow$ can decrease by at most one in a single iteration. We consider the process until $X_t^\leftarrow$ drops below $\lceil n(r-1)/8\rceil$. For $i\in[\lceil n(r-1)/8\rceil..\lceil n(r-1)/4\rceil]$, let $p_i$ denote the probability to decrease the value of $X_t^\leftarrow$ from $i$ to $i-1$ per iteration. Using Algorithm~\ref{alg:modified_semo}, the probability to select the desired individual whose $f_1$-value equals to $i$ is exactly $1/(n(r-1)+1)$. Hence, we know that $ p_i \le \frac{1}{n(r-1)+1}.$
    Since $i\ge n(r-1)/8$, at least $\lceil i/(r-1)\rceil \ge n/8$ entries are strictly positive, implying the lower bound
    \[
    p_i \ge \frac{1}{n(r-1)+1}\cdot \frac{n/8}{n}
    = \frac{1}{8n(r-1)+8}.
    \]
    Hence, the waiting time spent on each value $i$ is stochastically dominated by a geometric random variable with success probability $p_i\in[\frac{1}{8n(r-1)+8}, \frac{1}{n(r-1)+1}]$. Let $T^\star$ denote the total time until $X_t^\leftarrow < \lceil n(r-1)/8\rceil$. We know that 
    \begin{align*}
        \mathbb{E}[T^\star]
        &\ge \sum_{i=\lceil n(r-1)/8\rceil}^{\lceil n(r-1)/4\rceil}
        \frac{1}{p_i} \ge
        \sum_{i=\lceil n(r-1)/8\rceil}^{\lceil n(r-1)/4\rceil}
        \bigl(n(r-1)+1\bigr)\\
        &\ge \frac{n(r-1)}{8}\cdot \bigl(n(r-1)+1\bigr) \ge \frac{n^2(r-1)^2}{8}.
    \end{align*}

    Let $s:=\sum_i 1/p_i^2$. Using the above bounds on $p_i$, we obtain
    \[
    s \le \bigg(\frac{n(r-1)}{8}+2\bigg)\cdot \bigl(8n(r-1)+8\bigr)^2
    \le 8\bigl(n(r-1)+16\bigr)^3.
    \]
    Applying Theorem~\ref{the:witt14} with
    $\lambda:=n^2(r-1)^2/16\le \mathbb{E}[T^\star]/2$ yields
    \[
    \Pr [T^\star\le \lambda]
    \le \exp\left(-\frac{\lambda^2}{2s}\right)
    \le \exp\left(-\Theta(n)\right).
    \]

    Combining this bound with the failure probability of the initialization step, we conclude that with probability at least $1-\exp(-\Theta(n))$, the value $X_t^\leftarrow$ remains at least $n(r-1)/8$ for all $t\in[0..n^2(r-1)^2/16]$. The same argument applies to $X_t^\rightarrow$.  
\end{proof}

\begin{lemmaNP}[\ref{lem:low_SEMO_3}]
    Consider Algorithm~\ref{alg:semo} optimizing \goneminmax, starting with a population size of at least $\frac{n(r-1)}{130}$ and the distance to the Pareto borders of at least $n(r-1)/8$. Then, the expected number of iterations until the population covers the whole Pareto front is $\Omega(n^2r(r+\log n))$.
\end{lemmaNP}
\begin{proof}[Proof of Lemma~\ref{lem:low_SEMO_3}]
    As defined before in the proof of Lemma~\ref{lem:low_SEMO_2}, let $d_{PF}$ denote the distance of the population to the Pareto borders. Let $ S:=\inf\{t\in\mathbb{N}\mid d_{PF}(P_t)=0\} $ be the first time when the distance to the Pareto borders becomes zero. Since covering the Pareto front requires reaching both borders, we know that $S$ is stochastically dominated by the runtime of the SEMO covering the entire Pareto front. 
    
    By symmetry, it suffices to analyze the time until $X_t^\leftarrow=0$. Recall that $X_t^\leftarrow := \min_{z\in P_t} f_1(z)$ and $d_{PF}(P_t)=\min\{X_t^\leftarrow, X_t^\rightarrow\}$. By assumption, initially we have that $X_0^\leftarrow \ge n(r-1)/8$. Due to unit-strength local mutation, $X_t^\leftarrow$ can decrease by at most one per iteration. 
    Let $T^\star$ be the first time when $X_{T^\star}^\leftarrow=0$. We decompose the process into two phases $ T^\star = T_1^\star + T_2^\star$,  where $T_1^\star$ denotes the time until $X_t^\leftarrow \le n$ and $T_2^\star$ the remaining time from $X_t^\leftarrow\le n$ until $X_t^\leftarrow=0$. For each $i\ge 1$, let $p_i$ denote an upper bound on the probability that in a given iteration the value $X_t^\leftarrow$ drops from $i$ to $i-1$ (conditioned on $X_t^\leftarrow=i$). Then the waiting time on value $i$ is stochastically bounded from below by a geometric random variable with success probability $p_i$. Consequently, the time needed to go from value $a$ down to value $b$ is then stochastically bounded from below by a sum of geometric random variables with success probabilities $p_a,p_{a-1},\dots,p_{b+1}$. We now estimate $p_i$ in the two phases.

    For the first phase, since $|P_t|\geq (n(r-1))/130$, the probability to select a specific individual is at most $130/(n(r-1))$. Further, once this individual is selected, the probability that $f_1$ decreases by 1 is at most~$1$. Therefore, $p_i \le \frac{130}{n(r-1)} $ for all $ i\in[n+1..n(r-1)/8].$ 
    Recall that $T_1^\star$ denotes the time to decrease $X_t^\leftarrow$ from $\lceil n(r-1)/8\rceil$ down to $n$. Using $p_i\le 130/(n(r-1))$ we obtain
    \begin{align*}
        \mathbb{E}[T_1^\star]
        &\ge \sum_{i=n+1}^{\lceil n(r-1)/8\rceil} \frac{1}{p_i}
        \;\ge\; \sum_{i=n+1}^{\lceil n(r-1)/8\rceil} \frac{n(r-1)}{130}\\
        &\ge \Big(\frac{n(r-1)}{8}-n\Big)\cdot\frac{n(r-1)}{130}
        = \frac{n^2(r-1)(r-9)}{1040}.
    \end{align*}

    When $X_t^\leftarrow=i\le n$, the probability to decrease $X_t^\leftarrow$ in one iteration is at most the probability to select a suitable parent times the probability to flip one of the changeable entries. This gives
    \[
    p_i \le \frac{130}{n(r-1)}\cdot \frac{i}{n}
    = \frac{130\,i}{n^2(r-1)}.
    \]
    Recall that $T_2^\star$ denotes the time to decrease $X_t^\leftarrow$ from $n$ to $0$. Using the upper bound on $p_i$, we obtain
    \begin{align*}
        \mathbb{E}[T_2^\star]
        &\ge \sum_{i=1}^{n} \frac{1}{p_i}
        \ge \sum_{i=1}^{n}\frac{n^2(r-1)}{130i}
        > \frac{n^2(r-1)\ln n}{130}.
    \end{align*}
    
    Combining these two phases, we have 
    \begin{equation*}
        T^\star = T_1^\star + T_2^\star=\Omega\big(n^2(r-1)\,(r+\ln n)\big).  \qedhere
    \end{equation*}
\end{proof}

\begin{theoremNP}[\ref{the:lower_bound_SEMO}]
    Consider Algorithm~\ref{alg:semo} optimizing \goneminmax. Then
    the expected number of iterations until the population covers the whole Pareto front is $\Omega(n^2r(r+\log n))$.
\end{theoremNP}

\begin{proof}[Proof of Theorem~\ref{the:lower_bound_SEMO}]
    From Lemmas~\ref{lem:low_SEMO_1} and~\ref{lem:low_SEMO_2}, we know that with probability $1-\exp(-\Theta(n))$, in some iterations Algorithm~\ref{alg:modified_semo} will reach a state where the population size is at least $\frac{n(r-1)}{130}$ and the distance of the population to the Pareto borders is at least $n(r-1)/8$. 
    From the definitions of Algorithms~\ref{alg:semo} and~\ref{alg:modified_semo}, it is easy to see that both algorithms behave identically when Step~\ref{stp:act} of Algorithm~\ref{alg:modified_semo} is activated, and Algorithm~\ref{alg:modified_semo} leaves the population unchanged when this step is not activated. 
    Therefore, Algorithm~\ref{alg:semo} also reaches the above state with probability $1-\exp (-\Theta(n))$, possibly in fewer iterations (which is not relevant at this time). Applying Lemma~\ref{lem:low_SEMO_3} starting from this state, we derive that the expected number of iterations until the population covers the whole Pareto front from this point on is $\Omega(n^2 r (r + \log n))$. This proves the theorem.
\end{proof}

\begin{theoremNP}[\ref{the:Theta_bound}]
    Consider optimizing \goneminmax via Algorithm~\ref{alg:semo} with the modification that it accepts only strictly better solutions. Then the expected number of iterations until the population covers the whole Pareto front is $\Theta(n^2r(r+\log n))$.
\end{theoremNP}

\begin{proof}[Proof of Theorem~\ref{the:Theta_bound}]
    We first prove the upper bound. For the unit-strength local mutation, the maximum value of $f_1$ in the population (and analogously of $f_2$) can increase by at most one per iteration. Since already reached Pareto front points will be maintained in all future generations, we know that the maximum $f_1$-value and the maximum $f_2$-value reached by $P_t$ will not decrease. Consequently, once the population contains both extreme Pareto-optimal solutions $\{r-1\}^n$ and $0^n$, all intermediate objective values along the Pareto front must have been generated and accepted, and the population covers the entire Pareto front. 
    Let $T_1$ denote the first time when the population contains $\{r-1\}^n$ and $T_2$ the first time when the population contains $0^n$. With $T$ being the first time when the whole Pareto front is covered, we have that 
    $\mathbb{E}[T]\le \mathbb{E}[T_1]+\mathbb{E}[T_2]$.

    We now first estimate $\mathbb{E}[T_1]$. Let $X:=\max \{f_1(x)\mid x\in P_t\}$. We measure the progress of the population towards one that contains $\{r-1\}^n$ by the potential function
    \begin{align*}
        g(P_t):=\sum_{i=1}^n \bigl(2^{r-1-x_i}-1\bigr),
    \end{align*}
    where $x$ denotes the individual such that $f_1(x)=X$. We note that this potential function is very similar (and clearly inspired by) the potential function that was used in the analysis of how the randomized local search heuristic optimizes the variant of the $r$-valued single-objective \onemax problem studied in \cite{DoerrDK18}. 
    Note that $g(P_t)=0$ if and only if $\{r-1\}^n \in P_t$.
    The potential decreases only if an offspring with a larger $f_1$-value replaces $x$ as the individual with maximum $f_1$-value. Otherwise, $g(P_{t+1})=g(P_{t})$.\footnote{Note that here is where we crucially exploit the fact that solutions are not replaced by different ones having the same objective value, as this would allow for changes of the potential function $g$ without any real progress.}
    Let $O:=\{i\in[1..n]\mid x_i=r-1\}$ be the set of entries that already have their maximum value. Note that the probability of selecting $x$ for mutation is at least $1/(nr-n+1)$. Conditioned on this event, the unit-strength local mutation increases the value of a given entry of $x$ with probability $\frac{1}{2n}$. For any entry $i\in[1..n]\setminus O$, increasing $x_i$ by one decreases the potential by 
    \begin{align*}
        (2^{r-1-x_i}-1)-(2^{r-2-x_i}-1)=2^{r-2-x_i}.
    \end{align*}
    All other mutations are rejected and do not affect the potential. 
    Therefore, we have that
    \begin{align*}
        &\mathbb{E}[g(P_t)-g(P_{t+1}) \mid P_t]
        =\frac{\sum_{i\in[1..n]\setminus O}2^{r-2-x_i}}{2n(nr-n+1)}\\
        &\geq \frac{\sum_{i\in[1..n]}\big(2^{r-1-x_i}-1\big)}{4n(nr-n+1)}
        = \frac{g(P_t)}{4n(nr-n+1)}.
    \end{align*}
    
    Note that the maximal possible value of the potential is less than $n2^{r-1}$. Applying the multiplicative drift theorem (Theorem~\ref{the:drift analysis}) with $\delta=1/(4n(nr-n+1))$, we obtain
    \begin{align*}
        \mathbb{E}[T_1]
        = O\bigg(\frac{\ln(n2^{r-1})}{\delta}\bigg)
        = O\left(n^2r(r+\log n)\right).    
    \end{align*}
    By symmetry, the same bound holds for $\mathbb{E}[T_2]$, and finally 
    $\mathbb{E}[T]\le \mathbb{E}[T_1]+\mathbb{E}[T_2]= O(n^2r(r+\log n))$.

    Regarding the lower bound, we note that the proof of the corresponding result for the SEMO accepting equally good solutions also applies here, since it only relies on the objective values of the solutions, not on their precise representation as elements of $[0..r-1]^n$.
\end{proof}

\end{document}